\documentclass{article}

\usepackage[nonatbib,final]{neurips_2020}

\usepackage[latin1]{inputenc}
\usepackage{hyperref}       
\usepackage{url}   
\usepackage{amsmath}
\usepackage{color}
\usepackage{amsfonts,enumitem}
\usepackage{amssymb}
\usepackage{graphicx,tikz}
\usepackage{booktabs,float}          
\usepackage{amsfonts}       
\usepackage{nicefrac}       
\usepackage{microtype}    
\usepackage{wrapfig}
\usepackage{multicol}

\usepackage{float}

\usepackage{algorithmic}
\usepackage[ruled,noend]{algorithm2e}
\usepackage[skins]{tcolorbox}
\usepackage{mathabx}

\newcommand{\pr}{\mathrm{\texttt{pr}}}

\newcommand{\zero}{\mathrm{zero}}
\newcommand{\relu}{\mathrm{relu}}
\newcommand{\crit}{\mathrm{crit}}
\newcommand{\conv}{\mathrm{conv}}
\newcommand{\cl}{\mathrm{cl}\,}

\newcommand{\Sing}{\mathrm{Sing}}

\newcommand{\PP}{\mathcal{P}}

\newcommand{\SSS}{\mathcal{S}}

\newcommand{\RR}{\mathbb{R}}
\newcommand{\NN}{\mathbb{N}}

\newcommand{\A}{\mathcal{A}}
\newcommand{\FF}{\mathcal{F}}

\newcommand{\selgrad}{\widehat\nabla}
\newcommand{\seljac}{\widehat J}

\newtheorem{theorem}{Theorem}
\newtheorem{lemma}{Lemma}
\newtheorem{proposition}{Proposition}

\newtheorem{definition}{Definition}

\newtheorem{remark}{Remark}
\newtheorem{claim}{Claim}

\newenvironment{proof}[1][]{\noindent {\bf Proof #1:\;}}{\hfill $\Box$}

\newcommand{\grad}{\mathrm{grad}\,}

\newcommand{\R}{\mathbb{R}}

\title{A mathematical model for automatic differentiation in machine learning}
\begin{document}
\author{%
   J\'er\^ome Bolte\thanks{Authors in alphabetical order.}\\
   Toulouse School of Economics \\
   Univ. Toulouse\\
   Toulouse, France
   \And
   Edouard Pauwels\\
    IRIT, CNRS\\
   Univ. Toulouse\\
   Toulouse, France
}
\date{\today}

\maketitle

\begin{abstract}
Automatic differentiation, as implemented today, does not have a simple mathematical model adapted to the needs of modern machine learning. In this work we articulate the relationships between differentiation of programs as implemented in practice and differentiation of nonsmooth functions. To this end we provide a simple class of functions, a nonsmooth calculus, and show how they apply to stochastic approximation methods. We also evidence the issue of artificial critical points created by algorithmic differentiation and show how usual methods avoid these points with 
probability one. 
\end{abstract}


\section{Introduction}

Optimization algorithms based on backpropagation oracles, and more generally on automatic or algorithmic differentiation (AD) \cite{speelpenning1980compiling,rumelhart1986learning}, are one of the most widely used training tools for modern learning architectures \cite{bottou2008tradeoffs,lecun2015deep,bottou2018optimization,chizat,davis2018stochastic,baydin2018automatic,castera2019inertial}. They often rely on popular  numerical implementations as TensorFlow or PyTorch \cite{abadi2016tensorflow,paszke2017workshop}. However, for nonsmooth, nonconvex losses, AD does not have  a stable theory  \cite{griewank2008evaluating,griewank2013stable,griewank2016lipschitz,barton2018computationally,kakade2018provably,griewank2019treating,griewank2020beyond,bolte2020conservative}, matching the actual practice. We wish to present a simple mathematical framework addressing this issue. Let us progressively explain our approach.

\subsection{What is backpropagation?}
\paragraph{Algorithmic differentiation acts on programs not on functions:}
\label{sec:autodiffOnPrograms}
To convey this fact we carry out a small experiment in TensorFlow \cite{abadi2016tensorflow} with the function $\relu \colon t \mapsto \max\{0,t\}$, see Appendix \ref{sec:suppRelu} for implementation details. Algorithmic differentiation is displayed in Figure \ref{fig:illustrAutodiff}, in particular, we have $\relu'(0) = 0$. Consider the two functions
\begin{align*}
    \relu_2 \colon t \mapsto \relu(-t) + t,\qquad \qquad \relu_3 \colon t \mapsto \frac{1}{2}(\relu(t) + \relu_2(t)).
\end{align*}
As mathematical functions on $\RR$ these are {\em equal to} $\relu$. However TensorFlow returns $\relu_2'(0) = 1$ and $\relu_3'(0)=1/2$ (Figure \ref{fig:illustrAutodiff}). Indeed, AD does not act on functions, but on their representations, i.e., on  programs. Different programs implementing the same function may provide different results, beyond numerical precision; we refer to this as the spurious behaviour of AD for nonsmooth functions\footnote{
The validity domain of AD is restricted in theory to smooth functions \cite{griewank2008evaluating}, yet it is  common practice  to use it for nonsmooth functions.}.
Let us explore this phenomenon further. The function $\zero \colon t \mapsto \relu_2(t) - \relu(t)$, outputs constantly $0$ but AD gives $\zero'(0) = 1$. More generally, one can modify the value of the derivative of a given function at prescribed arguments (Figure \ref{fig:illustrAutodiff}). This may generate artificial critical points; for instance  $x\to x-\zero$ is the identity but its derivative at $0$ according to AD is $0$. 

This discussion was limited to univariate functions, but these pathologies grow in size and in complexity when occurring in higher dimensions. Besides, as the ``compositional depth" of functions increases the phenomenon gets more complex, making the geometry of artificial point difficult to grasp.
\begin{figure}[t]
    \centering
    \includegraphics[width=.9\textwidth]{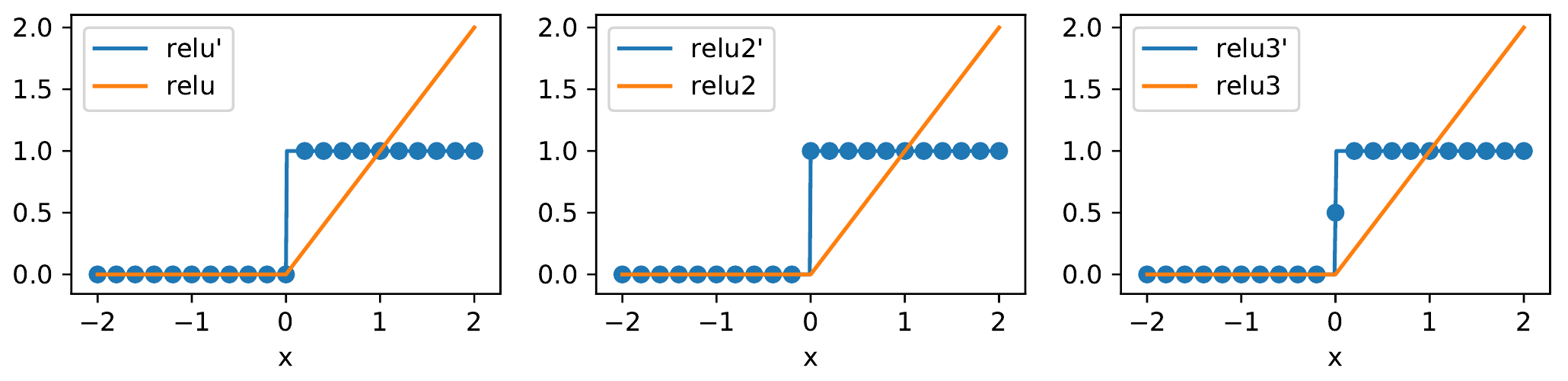}\\
    \includegraphics[width=.9\textwidth]{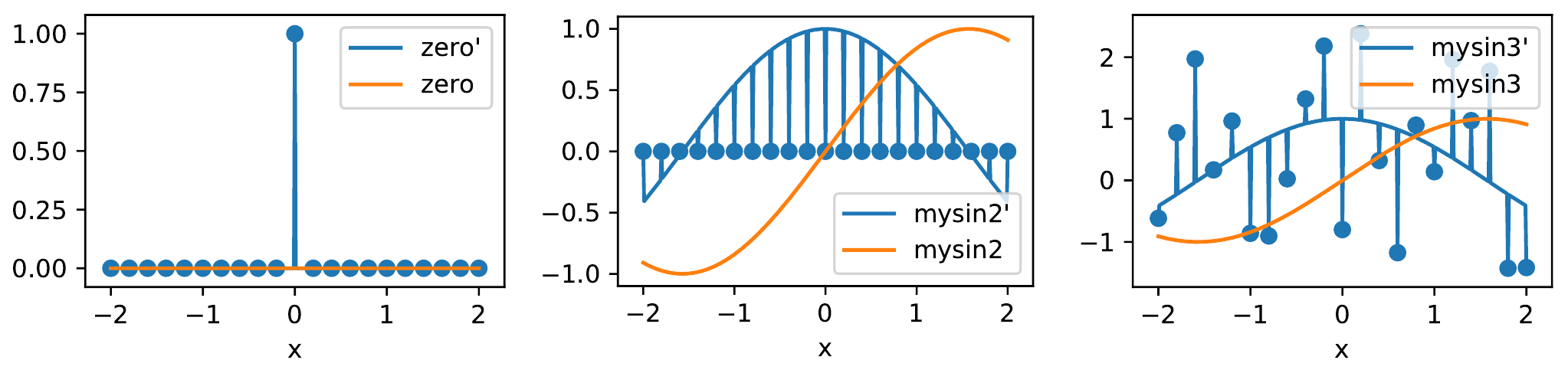}
    \caption{Top: AD applied to $\relu$ and two different implementations of the same function. Bottom: Algorithmic differentiation of a constant function, creation of artificial critical point or arbitrary derivatives at prescribed arguments for the sine function.}
    \label{fig:illustrAutodiff}
\end{figure}

\paragraph{Canonical surjection between programs functions:}
Numerical programs combine basic mathematical functions within an algorithm and return an output. This can be understood in two ways:
\begin{itemize}\itemsep-.1em
    \item Computer science: it is a sequence of instructions with numerical inputs-outputs,
    \item Mathematics: the program is a function\footnote{In the usual mathematical sense.} of its arguments.
\end{itemize}
It is tempting to identify both, but functions can be represented by different programs. This defines a surjection $\mathcal{F}$ mapping a program to a function (in the class of functions ``accessible through coding"). 

\paragraph{Algorithmic differentiation:} As presented above, AD is an operation on programs, $\mathcal{A}$ which takes as argument a program and returns a program with the same input variables. This operation can be ``pushed'' to the space of functions using the canonical surjection $\mathcal{F}$. Remar\-ka\-bly, if we restrict ourselves to programs $\mathcal{P}$ which only smoothly combine smooth functions, then we have the following fundamental relation, depicted in Figure \ref{fig:diagrams}:
\begin{align}
    \mathcal{F} (\mathcal{A}(\mathcal{P})) = \nabla \mathcal{F}(\mathcal{P}).
    \label{eq:fundamentalRelation}
\end{align}
In other words, algorithmic differentiation of a program which smoothly combines smooth functions, is equivalent, through the canonical surjection, to derivation.

\begin{figure}[H]
    \center
    \includegraphics[width=\textwidth]{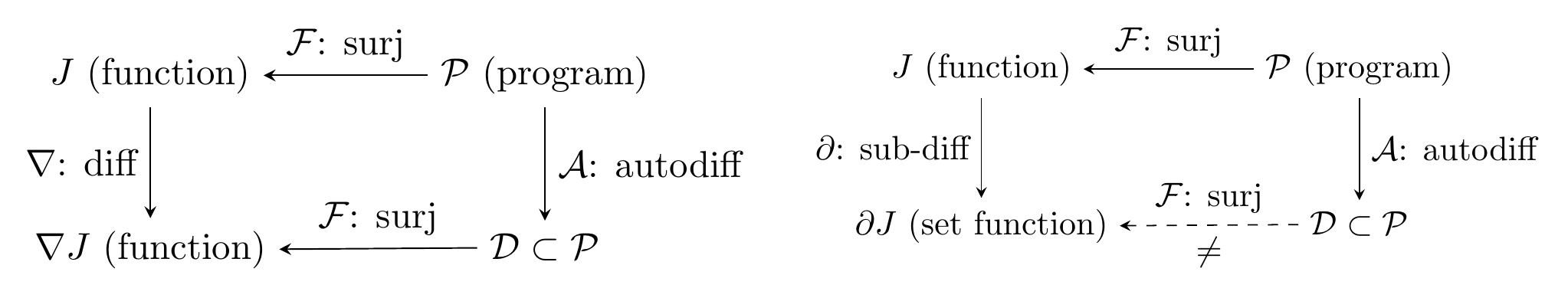}
    \caption{Left: Algorithmic differentiation applied to programs combining smooth functions in a smooth way, the diagram commutes. Right: Algorithmic differentiation in nonsmooth settings, connection with known notion of generalized derivative is much less clear.}
    \label{fig:diagrams}
\end{figure}

However practitioners use AD and backpropagation  beyond smooth programs with nonsmooth elementary functions or program branching for instance. Can we find a proper operational  interpretation of this widespread practice?

\paragraph{Algorithmic differentiation cannot be represented through a variational operator}
At first, it is tempting to simply use AD to induce a differential operator on functions generalizing classical differentiation. This operator, say $\partial^{A}$, should:
\begin{itemize}
    \item[(a)] encompass the outputs of algorithmic differentation for all functions
    \item[(b)] be such that $0$ is an element of $\partial^{A}(\relu)$ at $0$. 
\end{itemize}
Unfortunately such an operator does not exist:
\begin{theorem}[Algorithmic differentiation does not induce an operator on functions]\label{th:noOperator}
There is no nontrivial operator on functions satisfying $(a)$ and $(b)$.
\end{theorem}

\subsection{Contribution and related work} 
We address  this impossibility result and provide a class of functions together with an operational nonsmooth differential calculus which is able to cope with spurious behaviours. 

\paragraph{Elementary selections and selection derivatives:} We introduce a new  class of  nonsmooth nonconvex functions, encompassing most objective functions met in machine learning, having appealing stability properties. This allows us to define simple  differential objects called selection derivatives. Selection derivatives turn out to have an operational calculus adapted to the analysis of many learning methods, as backpropagation or stochastic first order methods. They thus provide an operational model to  capture nonsmooth AD as implemented in current numerical software.

\paragraph{Algorithmic differentiation, algorithms} This framework allows to formalize properly the relationships between, functions, algorithmic differentiation and capture the corresponding notion of critical points as met in practice. These  characterize the set of attractors (limit points) for  stochastic approximation algorithms based on nonsmooth backpropagation  \cite{robbins1951stochastic,benaim1999dynamics,kushner20003stochastic,benaim2005stochastic,borkar2009stochastic}.  
It is important to stress that these attractors, which  models sharply the whole scope of AD-induced  stationnarity,  are different from the traditional notions as Clarke criticality   \cite{clarke1983optimization,rockafellar1998Variational,davis2018stochastic}. 
This is described in Theorems~\ref{th:tensorflowConservativeField} and \ref{th:convergenceGrad}.

\paragraph{Avoidance of traps:} As sketched above and in the introduction AD produces artificial critical points, i.e. stationary points which are not Clarke critical. These points have a parasitic nature which could be detrimental to training purposes, were they met in practice. We show that randomly initialized mini-batch stochastic gradient method do not lead to  artificial critical points (Theorem~\ref{th:trapAvoidance}). This result applies to modern machine learning software libraries based on AD \cite{abadi2016tensorflow,paszke2017workshop}, seen as performing operation over the reals, without any modification. Although AD may have unpredictable behavior in nonsmooth contexts, both theoretically and numerically, this result justifies theoretically that the practical impact is somewhat negligible in the context of common machine learning usage.

\paragraph{Related work:} Spurious behaviour of AD in nonsmooth context has been investigated in \cite{griewank2008evaluating,griewank2013stable,griewank2016lipschitz,griewank2016first,barton2018computationally,kakade2018provably,bolte2020conservative}. In particular, \cite{griewank2016first,kakade2018provably} considers qualification conditions allowing to construct AD algorithms which compute proper Clarke subgradients  \cite{clarke1983optimization,rockafellar1998Variational,davis2018stochastic}.  However qualification is extremely hard to check and almost impossible to enforce in practice. 
Let us also mention \cite{barton2018computationally} which uses the notion of lexicographic derivatives, but, at this day, algorithmic computations are limited to forward mode for the moment which is of little use in machine learning.

\cite{griewank2008evaluating,griewank2013stable,griewank2016first,griewank2016lipschitz,griewank2019treating,griewank2020beyond} use settings closer to ours. Piecewise 
smooth functions, selection derivatives and their variational properties are extensively described in \cite{scholtes2012introduction}. Our approach differs because we adopt more stringent definitions and rigidity assumptions, which allows in turn for much stronger properties.  For instance, we fully treat backward algorithmic differentiation which is the most useful tool in machine learning.
 
Altogether, our contribution is an accessible and elementary framework for the conservative fields recently introduced in \cite{bolte2020conservative}, without explicitly requiring the introduction of semialgebraic geometry and o-minimal structures \cite{dries1996geometric,Cos99}.

Stochastic approximation algorithms \cite{robbins1951stochastic,benaim1999dynamics,kushner20003stochastic,benaim2005stochastic,borkar2009stochastic} are widely used in machine learning contexts \cite{rumelhart1986learning,bottou2008tradeoffs,moulines2011nonasymptotic,lecun2015deep,bottou2018optimization,chizat,castera2019inertial}. For example \cite{davis2018stochastic} describes asymptotics of stochastic subgradient algorithms in nonsmooth, nonconvex settings. In contrast, we do not assume access to subgradients and instead explicitly model the behaviour of AD in optimization contexts. Our convergence results are based on \cite{bolte2020conservative}, complemented by a new result on 
``the avoidance of critical traps'' in the line of \cite{bianchi2020convergence} in the context of long run convergence.

\noindent
{\bf Notations} The ambient space is Euclidean $\RR^p$. For each $k$, $e_k$ is the $k$-th vector of the canonical basis. We use $D \colon \RR^m \rightrightarrows\RR^q$ for set valued functions, \textit{i.e} functions from $\RR^m$ to the subsets of $\RR^q$. The convex hull of $A \subset \RR^p$ is denoted by $\mathrm{conv}(A)$. All proofs are postponed to the Appendix.
\section{Basic piecewise differentiable functions and selection gradient}
\label{sec:basicPDfunction}
We introduce a simple but vast class of functions that model rigorously the  machine learning models and losses for applications such as deep learning.
\begin{definition}[Elementary (log-exp) functions]{\rm 
				{\em Elementary (log-exp) functions}  are functions on $\RR^p$ described by a finite compositional expression involving basic operations, $+,-,\times, /$ as well as affine mappings, exponential and logarithms, inside their domain of definition.  We denote by $\mathcal{E}$ the set of elementary functions in any dimension~$p$.
				\label{def:elemFunction}}
\end{definition}
Examples include polynomials, logistic loss, boosting loss, Gaussian likelihood.
Observe that  the corresponding functions are  $C^\infty$ smooth on their open domains. Note also that if log and exp are not present we obtain the field of rational functions. See Remark~\ref{r:elementary} in Appendix \ref{sec:auxiliaryResults}.

\begin{definition}[Elementary  index]
				$s \colon \RR^p \mapsto \left\{ 1,\ldots, m \right\}$ is an elementary (log-exp) index if the set $\left\{ x \in \RR^p,\, s(x) = i  \right\}$ is the solution set of a finite number of inequalities and equalities involving elementary functions on $\RR^p$. 
				The set of such functions is denoted by  $\mathcal I$  (for any input dimensions $p$).
				\label{def:elemIndex}
\end{definition}
\textbf{Examples:} The Heaviside function, the index of the largest or $k$-th largest element in a vector, the sign pattern of a vector in $\RR^p$ which is indexed by integers from $1$ to $2^p$.

\begin{definition}[Elementary  selection]{\rm
    Let $f \colon \RR^p \mapsto \RR$ be continuous. We say that $f$ has an {\em elementary (log-exp)  selection} $(s, f_1, \ldots, f_m)$ if $s \colon \RR^p \mapsto \left(1,\ldots, m\right)$ is an elementary index in $\mathcal{I}$ and for $i = 1 \ldots, m$, $f_i\colon \RR^p \mapsto \RR$ are elementary functions in $\mathcal{E}$, such that for all $x \in \RR^p$,
    \begin{align}
        f(x) = f_{s(x)}(x).
        \label{eq:C1Selection}
    \end{align}
    The $m+1$-uplet $(s, f_1, \ldots, f_m)$ is a {\em representation} of $f$, and $f$ admits an {\em elementary (log-exp) selection}.  
		The class of such functions is denoted by $\SSS_{\log\exp}$ or simply here $\SSS$. This extends to functions from $\RR^p$ to $\RR^m$ by applying a coordinatewise definition with a common elementary index.
    \label{def:C1selection}}
\end{definition}
Observe that the representation is {\em never} unique, both in $s$ and in the sequence $f_1,\ldots,f_m$.  The ReLU, hinge loss, maximal entry,  $k$-th largest entry functions are elementary selections. Note also that continuity is part of the definition.

\begin{proposition}[Stability of $\SSS$ by $\circ,+,\times$]\label{p:stability}
             The class $\SSS$ of elementary selections is stable by composition, sum and product.
\end{proposition}
The class $\SSS$ is close to the one of piecewise $C^k$ functions, see e.g \cite{scholtes2012introduction}, but it is also much more disciplined since indices and functions are required to satisfy strong ``log-exp" rigidity assumptions.

\subsection{Selection derivative}
Functions in $\mathcal{S}$ can be associated with a flexible notion of generalized derivative based on the selection structure of the underlying function.
\begin{definition}[Selection gradient]{\rm
				(i) Let $f \colon \RR^p \mapsto \RR$,  in $\mathcal{S}$ with selection $(s, f_1, \ldots, f_m)$. We set the {\em selection derivative of $f$ with respect to $s$} to be
				\begin{align}\label{e:deriv}
								\selgrad^s f \colon x \mapsto \nabla f_{s(x)}(x).
				\end{align}
				This extends to multivariate outputs by applying the definition coordinatewise, which leads to a notion of a {\em selection Jacobian} denoted by $\seljac^s$.\\
				(ii) Given a function $f\in \SSS$, a {\em selection derivative} is a derivative of the form \eqref{e:deriv} for a given representation. In that case a selection derivative of $f$ is merely denoted by $\selgrad f$.
				\label{def:derivative}}
			\end{definition}

\textbf{Example:} Set for all $x \in \RR$, $f_1(x) = 0$, $f_2(x) = x$ and $s(x) = 1$ for $x \leq 0$ and $s(x) = 2$ for $x>0$. This this defines the $\relu$ function and its selection derivative at $0$ is $0$. See more in Appendix \ref{sec:suppRelu}.

\textbf{Remark:} (a) $\selgrad f$ is different from any known notion of subgradient. Set for all $x \in \RR$, $f_1(x) = 0$, $f_2(x) = x$ and $s(x) = 1$ for $x \neq 0$ and $s(0) = 2$. This defines a  elementary selection  for the null function however, $\selgrad^s f(0) = 1$. This is the $\zero$ function of the introduction.\\
(b) This formalizes what one would obtained by differentiating a code with all decision branches frozen and hence represents the numerical output of AD (see \ref{sec:autodiff}). Note that one  only needs one branch and do not need to explore all possible outcomes, avoiding combinatorial explosion.

The properties of selection derivatives might seem too liberal at first sight and too disconnected from the original function, but this is not the case as shown below.
\begin{proposition}[Integration along segments]
				Let $f \colon  \RR^p \mapsto \RR$ be in $\SSS$, with elementary selection $(s, f_1, \ldots, f_m)$. Then $f$ is locally Lipschitz and for all $y,x$ in $\RR^p$.:
				\begin{align*}
										f(y)-f(x)=	\int_{0}^1 \left\langle y - x, \selgrad^s f(x + t(y-x)) \right\rangle dt
				\end{align*}
				
				\label{prop:SCintegration}
\end{proposition}

\begin{proposition}[Gradient almost everywhere]\label{gevery}
				Let $f \colon  \RR^p \mapsto \RR$ be in $\SSS$, with elementary selection $(s, f_1, \ldots, f_m)$. There exists sets $U_1,\ldots,U_N$ with nonempty interior such that $\bigcup_{i=1}^N \mathrm{cl}(U_i) = \RR^p$ and for each $i=1$, and for all $x$ in the interior of $U_i$, $\selgrad^s f(x) = \nabla f(x)$.
				Furthermore, the $U_i$ are solution sets of equations and inequalities involving functions in $\mathcal{E}$.
				\label{prop:SGradientAE}
\end{proposition}

\textbf{Remark:} Although less transparent, Proposition \ref{prop:SCintegration} is not a consequence of Proposition \ref{prop:SGradientAE}.
Both results crucially rely on the rigidity of elementary functions in $\mathcal{E}$ (Definition~\ref{def:C1selection}), not only on their piecewise smoothness. This a central novelty of our approach.

\subsection{A calculus for selection derivatives}

One has an unusual differential calculus: although it does not involve the linearity of some (sub)differential operator, the selection derivative of a sum gives a sum of selection derivatives provided that the selection is refined. 

\begin{proposition}[Chain rule]
				Let $F \colon \RR^{p_1} \mapsto \RR^{p_2}$ such that each of its coordinate $f_i$, $i = 1 \ldots p_2$, is in $\mathcal{S}$ and $g \colon \RR^{p_2} \mapsto \RR$, $g \in \mathcal{S}$. Consider a  selection Jacobian for $F$,   $\seljac_F \colon \RR^{p_1} \mapsto \RR^{p_2 \times p_1}$ 
				\begin{align}
								x &\mapsto \begin{pmatrix}\selgrad {f_1}(x)^T \\
										\vdots\\
										\selgrad {f_q}(x)^T
								\end{pmatrix}
							\label{eq:chainRule}\end{align}
				Then $g \circ F \in \mathcal{S}$ and the function $x \mapsto \seljac_F(x)^T \selgrad g(F(x))$ is a selection derivative for $g \circ F$.
				\label{prop:compositionGeneralizedDerivatives}
\end{proposition}

Proposition \ref{prop:compositionGeneralizedDerivatives} extends readily to the case when the outer function $g$ is multivariate. For example, we  have a sum rule $\selgrad (f+g)=\selgrad f+\selgrad g$ for full-domain functions $f,g$ in~$\SSS$. Indeed, if $F_1$ and $F_2$ are elementary selections then $F_1 \circ F_2\in \SSS$  and 
\begin{align}
    \seljac_{F_1 \circ F_2} = (\seljac_{F_1} \circ F_2) \times \seljac_{F_2}.
    \label{eq:chainRule}
\end{align}

\section{Programs and elementary selections}
\label{sec:programs}
Numerical programs encode numerical functions by combining elementary functions using a predecessor relation which models program  execution. 
In what follows, $m$ can be seen as an estimate of the memory footprint of a program\footnote{We consider programs which do not overwrite values in memory}, while $p$ and $q$ the number of inputs and outputs respectively. 

Given positive integers $m \geq p+q$, a {\em predecessor relation} is a set valued map $\pr \colon \left\{ 1,\ldots,m \right\} \rightrightarrows \left\{ 1, \ldots, m \right\}$ such that 
\begin{itemize}
				\item For $i \in \left\{ 1,\ldots,m \right\}$ and $j \in \pr(i)$, $j < i$. \qquad $\bullet$ For $i \in \left\{ p+1,\ldots,m \right\}$, $\pr(i)$ is nonempty.
\end{itemize}

\begin{minipage}{.55\textwidth}
A predecessor relation induces a partial order on the set of integers from $1$ to $m$ and hence can be represented by a directed acyclic graph \cite[Theorem 9.4.9]{lehman2010mathematics}.
Given $(p,q,m)$ and a predecessor relation $\pr$, a elementary function sequence $\mathcal{G} = (g_i)_{i=p+1}^m$ is a set of functions such that $g_i \colon \RR^{|\pr(i)|} \mapsto \RR$, and $g_i \in \mathcal{S}$, for all $i = p+1, \ldots, m$.  A program $P$ is then given by the data
$		P = \left(p,q,m, \pr, \mathcal{G} \right),	\label{eq:program}
$ while its {\em evaluation} is described in Algorithm \ref{alg:algof}. We denote by  
$\mathcal{P}$ the set of programs, and $\mathcal{P}_{p,q}$ when input-output dimensions have to be made  explicit. 

By definition a program encodes a function, but the representation is not unique. We express this fact below through the canonical surjection $\mathcal F$ of the introduction. 
\end{minipage}\qquad
\begin{minipage}{.4\textwidth}
\begin{algorithm}[H]
  \caption{Program evaluation}
	\label{alg:algof}
	\textbf{Program data:} $p,q \geq 1$, $m \geq p+q$, $\pr$ a predecessor relation, $\mathcal{G} = (g_i)_{i=p+1}^m$ an adapted function sequence.
  \begin{algorithmic}[1]
  \item[\textbf{Input:}] $x=(x_1, \ldots x_p)$
    \FOR{$k=p+1,p+2,\ldots m$}
    \STATE $x_k = g_{k} (x_{\pr (k)})$ where $x_{\pr (k)} = \left( x_i \right)_{i \in \pr (k)}$.
    \ENDFOR
		\item[\textbf{Return:}] $y := (x_j)_{j = m-q + 1}^m$.
\end{algorithmic}
\end{algorithm}
\end{minipage}
The following proposition illustrates the fact that practitioners {\em implicitly} implement selection functions when writing programs.
\begin{proposition}[Programs represents elementary selections]
   Through its input-output correspondence each program $P$ of the form \eqref{eq:program} induces  a function which is an elementary selection.     
   In other words $\mathcal{F}(P) \in \mathcal{S}.$
\end{proposition}

\section{Algorithmic differentiation and a variational model}
\label{sec:autodiff}

Algorithmic differentiation is based on the idea of propagating infinitesimal variations in a program $P$ through the chain rule,  either forward or backward. 

\begin{algorithm}[H]
  \caption{Algorithmic differentiation computes selection gradients}
	\label{alg:autodiff0}
	\textbf{Program data:} $p \geq 1$, $m \geq p+1$, $\pr$ a predecessor relation, $\mathcal{G} = (g_i)_{i=p+1}^m$ an adapted function sequence. \\
	\textbf{Input:} variables $(x_1, \ldots x_m)$ computed by Algorithm \ref{alg:algof}, $d_i = (d_{i}[j])_{j=1}^{|\pr (i)|} =  \selgrad {g_i}(x_{\pr (i)})$, $i = p+1 \ldots m$
	\vspace{.1in}
   
   \begin{minipage}[h]{0.45\linewidth}
   \begin{algorithmic}[1]
    \STATE \textbf{Forward mode:}
    \STATE Initialize: $
    \frac{\partial x_k}{\partial x} = e_k $,\\ $k = 1,\ldots, p$.
    \FOR{$k= p+1, \ldots m$}
    \STATE 
    \[
    \frac{\partial x_k}{\partial x} = \sum_{j \in \pr (k)} \frac{\partial x_j }{\partial x} d_{k}[j]
    \]
    \mbox{where $x=(x_1,\ldots,x_p).$}
    \ENDFOR
    \item[\textbf{Return:}]
$\frac{\partial x_m}{\partial x}$.
  \end{algorithmic}
  \end{minipage}
  \begin{minipage}[h]{0.45\linewidth}
     \begin{algorithmic}[1]
     \STATE \textbf{Backward mode:}
    \STATE Initialize: $ v = e_m$
    \FOR{$t= m, \ldots p+1$} 
    \FOR{$j \in \pr (t)$}
        \STATE Update coordinate $j$ of $v$:
    \[
    v[j] \mathrel{:=}  v[j] + v[t] d_{t}[j]
    \]
    \ENDFOR
    \ENDFOR
    \item[\textbf{Return:}]
$\left(v[1], v[2], \ldots, v[p]\right) $.
  \end{algorithmic}
  \end{minipage}
\end{algorithm}

Consider Algorithm \ref{alg:algof}, and assume for simplicity that $q=1$. The program can be seen as the implementation of $m-p$ successive transformations on $\RR^m$, of the form
\begin{align*}
        G_k \colon \RR^m &\mapsto\RR^m \\
        x &\mapsto x + e_k(g_k(x_{\pr (k)}) - x_k),
\end{align*}
for $k = p+1, \ldots, m$ which belong to $\mathcal{S}$. Algorithm \ref{alg:autodiff0} combines gradients dynamically along two modes: forward or backward. 
Let us describes these two forms.

Fix $x \in \RR^m$. After applying Algorithm \ref{alg:algof}, for each $k$, let $d_k \in \RR^m$ be the selection gradient $\selgrad g_k(x_{\pr (k)})$, appending $0$ to non dependant coordinates. A selection Jacobian of $G_k$ (at $x$) is given by
\begin{align*}
    \seljac_{G_k}  = I - e_k e_k^T + e_k d_k^T
\end{align*}
Denote by $J_p \in \RR^{m \times p}$, the matrix whose entries are $0$, except for diagonal entries which are $1$. In Algorithm \ref{alg:autodiff0}, the forward mode computes 
\begin{align*}
    e_m^T  \seljac_{G_m}  \ldots  \seljac_{G_{p+1}} J_p = e_m^T  \left(I - e_m e_m^T + e_m d_m^T \right)  \ldots   \left(I - e_{p+1} e_{p+1}^T + e_{p+1} d_{p+1}^T \right) J_p
\end{align*}
which is a selection Jacobian thanks to the chain rule in \eqref{eq:chainRule}.
On the other hand the backward mode computes
\begin{align*}
J_p^T  \left(I + d_{p+1}e_{p+1}^T \right) \ldots \left(I + d_{m}e_{m}^T \right) e_m.
\end{align*}
This quantity turns out to be the same as the one computed by the forward mode thanks to:
\begin{lemma}
		Let $p,m \in \NN$, $0<p<m$. Assume that for $i = p+1,\ldots, m$ we have $d_i \in \RR^m$. Then we have 
		\begin{align}
					P_p \left(I - e_{p+1} e_{p+1}^T + d_{p+1}e_{p+1}^T \right) \ldots \left(I - e_{m} e_{m}^T + d_{m}e_{m}^T \right) = P_p \left(I + d_{p+1}e_{p+1}^T \right) \ldots \left(I + d_{m}e_{m}^T \right) 
					\label{eq:backpropAlgebra}
		\end{align}
		where $I \in \RR^{m \times m}$ is the identity matrix and $P_p \in \RR^{m \times m}$ denotes the projection on the first $p$ coordinates.
		\label{lem:algebraBackprop}
\end{lemma}

Denote by  $\mathcal{A}:\PP_{p,1}\to \PP_{p,p}$ the algorithmic-differentiation operator. This establishes the following fundamental fact which is at the root of this work. This result asserts that practitioners {\em implicitly} implement selection derivatives when writing numerical programs and calling forward or backward AD on these programs.
\begin{theorem}[Algorithmic differentiation outputs a selection gradient]\label{t:autodiffgrad}
Algorithmic differentiation of a given program, i.e.,  $\mathcal{A}(P)$, outputs a selection derivative of the underlying numerical function. In other words there exists a representation of the numerical function $\FF(P)$ with elementary index $s$ such that:
$$\FF(\mathcal{A}(P))=\selgrad^s \FF(P).$$
\end{theorem}

\section{Algorithmic differentiation at work}

\subsection{Selection derivatives, conservative fields and Clarke subgradient}

The asymptotic study of first-order optimization methods implies limiting processes and necessitates thus the introduction of graph closed operators. Given a representation for $f$, we may construct such a convex-valued mapping pointwise as follows\footnote{Minimality relates to the {\em representation of the function}, not the function itself. This is the minimal convex-valued operator, constructed pointwise and guaranteed to be graph-closed.}. 
\begin{definition}[Representation minimal operator]\label{def:tensorflowConservativeField}
			{\em	Let $f \in \mathcal{S}$ with elementary selection $(s, f_1, \ldots, f_m)$. For any $x \in \RR^p$, set $I(x) = \left\{ i\in\left\{ 1,\ldots,m \right\}, \, f(x) = f_i(x) \right\}$. The index closure of $\selgrad^s f$ is given by the set valued map
				\begin{align*}
				D^s_f \colon \RR^p &\rightrightarrows \RR^p \\
								x &\rightrightarrows \mathrm{conv}\left( \left\{ \nabla f_i(x),\, i \in I(x) \right\} \right).
				\end{align*}
				where the double arrows express that the map has values in subsets of $\RR^p$, much like subgradients, and $\conv$ denotes the convex hull.
}
\end{definition}

The role of $D^s_f$ is to capture all possible outputs of AD including all possible program branches. Of course, due to combinatorial explosion, this quantity is intractable in practice. Its introduction here is only instrumental, we do not use it in algorithms, we just need to access one of its element, for example using a selection derivatives, obtained from AD.
A point $x$ satisfying $0 \in D^s_f(x)$ is called a {\em selection critical point}. We will often drop the index $s$ and write $D_f=D^s_f$. 

The two following results highlight crucial properties of $D_f$ in terms of optimization, they again rely on the rigidity constraint of elementary functions.

\begin{theorem}\label{th:tensorflowConservativeField}
				Let $f \in \mathcal{S}$ with  elementary selection  $(s, f_1, \ldots, f_m)$ and $D_f$ be as in Definition \ref{def:tensorflowConservativeField}. Then $D_f$ is conservative for $f$, that is for all absolutely continuous curves $\gamma \colon [0,1] \mapsto \RR^p$, for almost all $t \in [0,1]$, $f\circ \gamma$ is differentiable and
				\begin{align*}
								\frac{d}{dt} f(\gamma(t)) = \left\langle v, \dot{\gamma}(t) \right\rangle,\qquad \forall v \in D_f(\gamma(t)).
				\end{align*}
\end{theorem}

The previous result generalizes Proposition \ref{prop:SCintegration} by allowing to integrate arbitrary selections along absolutely continuous curves.
This connects our work to the general setting of \cite{bolte2020conservative}, note that $D_f$ has a closed graph thanks to Proposition \ref{prop:Dcontinuity} in Appendix \ref{sec:auxiliaryResults}.

In \cite{scholtes2012introduction}, the author considers the {\em essential index set}, for each $x \in \RR^p$,
\begin{align*}
    S_E(x) = \left\{ i\in\left\{ 1,\ldots,m \right\}, \, x\in \mathrm{cl}(\mathrm{int}(\{y,\,f(y) = f_i(y)\})) \right\} \subset S(x).
\end{align*}
Considering Definition \ref{def:tensorflowConservativeField} with $S_E(x)$ instead of $I(x)$ leads to the Clarke subgradient, which can also be defined as 
$$\partial^cf(x)=\mathrm{conv}\{d\in \R^p:\exists x_k \in \Delta_f, x_k\to x, \nabla f(x_k)\to d\}$$
where $\Delta_f$ is the dense set of differentiability points of $f$.  While $I(x)$ can be computed pointwise (check finitely many equalities), it might be very hard to check membership in $S_E(x)$ without restrictive qualification conditions on programs \cite{kakade2018provably}.

\paragraph{Illustration with ReLU and sorting:} 
(a) Set for all $x \in \RR$, $f_1(x) = 0$, $f_2(x) = x$, $s(x) = 1$ for $x \leq 0$ and $s(x) = 2$ for $x>0$. This  is $\relu$. In this case $D_f = \partial \relu$, the convex subgradient.\\
(b)  
Let  $F \colon \RR^p \mapsto \RR^p$ to be the sorting function which associates to $x$ a vector $Px$ where $P$ is any permutation such that $Px$ belongs to the set of vectors which values are sorted in descending order coordinatewise. $F$ obviously has an elementary selection  and the construction which we have proposed leads to
\begin{align*}
    D_F \colon x &\mapsto \conv \left\{ P \in \Delta, \quad Px = F(x) \right\},
\end{align*}
where $\Delta$ denotes the set of permutation matrices of size $p \times p$. Then $D$ is a conservative mapping for $F$ and it actually corresponds to the Clarke Jacobian.

\subsection{Convergence of gradient type algorithm and criticality of limit points}

Optimization processes in learning are supposed to provide at least a critical point $x$ of the loss, i.e. a point satisfying   $0 \in \partial^cf(x)$. When using AD  one enlarges the definition of criticality into 
$0 \in D_f(x)$  and {\em artificial critical points} appear, they satisfy $0 \not \in \partial^c f(x)$ and $0 \in D_f(x)$. Artificial critical points could possibly trap the optimization process in strongly non-optimal situations, we thus have to determine if they have an impact on learning phases. 

We consider the problem
\begin{align}
    \min_{x \in \RR^p} J(x) = \frac{1}{n} \sum_{i=1}^n f_i(x)
    \label{eq:mainProblem}
\end{align}
where $f_i \colon \RR^p \mapsto \RR$, $f_i \in \mathcal{S}$, $i=1,\ldots, n$. We consider the following algorithm, given $x_0 \in \RR^p$, a sequence of positive step sizes $(\gamma_k)_{k \in \NN}$ and a sequence of \textit{iid} indices $(I_k)_{k\in \NN}$ taken uniformly in the nonempty subsets of $\left\{0,\ldots, n\right\}$,
\begin{align}
    x_{k+1} = x_k - \gamma_k \selgrad f_{I_k}(x_k) \mbox{ where }f_I=\frac{1}{|I|}\sum_{i\in I}f_i, \:I\subset \{1,\ldots,n\}.
    \label{eq:gradientDescent}
\end{align}
Note that as discussed in Section \ref{sec:autodiff} selection derivatives can be computed by AD if $f_i$ are given by the data of numerical programs as in \eqref{eq:program}, and could be far from usual notions of subgradients. Hence this algorithm models explicitly the training of a nonsmooth deep network using existing backpropagation implementations. Note that $J \in \mathcal{S}$ and that $1/n \sum_{i=1}^n \selgrad f_i$ is a selection gradient for $J$ as stated in Proposition \ref{prop:compositionGeneralizedDerivatives}, denote by $\selgrad J$ this quantity and $D_J$ the corresponding set valued field (Definition \ref{def:tensorflowConservativeField}). The following result illustrates that selection critical points are the only attractors for the recursion and that generically such attractors are actually Clarke critical. The first result stands on the theory developed in \cite{benaim2005stochastic}. The second parallels developments in \cite{bianchi2020convergence} in the context of long run convergence. The spurious behaviour illustrated in Figure \ref{fig:illustrAutodiff} does not affect asymptotics, for typical initialization.
\begin{theorem}[Convergence and insignificance of artefacts]
    \label{th:convergenceGrad}
    Let for all $k$, $\gamma_k = c \alpha_k$ where $c \in (0,1]$ and $\alpha_k=o(1/\log k)$ and $K \subset \RR^p$ be open.
Assume that for all $c \in (0,1]$ and all $x_0 \in K$ the sequence in \eqref{eq:gradientDescent} is bounded almost surely. 
\begin{itemize}
    \item For all $x_0 \in K$, almost surely, $J(x_k)$ converges as $k$ tends to infinity and all accumulation points, $\bar{x}$, of $(x_k)_{k \in \NN}$ are selection critical points: $0 \in D_J(\bar{x})$.
    \item For almost all $c \in (0,1]$, almost all $x_0 \in K$, and almost surely, any accumulation point, $\bar{x}$, of $(x_k)_{k \in \NN}$ is Clarke critical: $0 \in \partial^cJ(\bar{x})$.
\end{itemize}

    \label{th:trapAvoidance}
\end{theorem}

\section{Conclusion}
The current work departs from existing approaches to  nonsmooth algorithmic differentiation  in a fundamental way. We propose to study the backward mode of AD, as implemented in machine learning, without any modification. Our theoretical results model thus AD  ``as is", and our focus is precisely on its unpredictable behavior  in a nonsmooth context, addressing an issue which is ubiquitous in machine learning. Our main contribution was to prove  that, in a stochastic optimization context, this spurious behavior is essentially harmless from a theoretical point of view, providing justifications for the use of AD outside of its original domain of validity in machine learning. 

We achieve our goal by modeling sharply common machine learning functions and their differentiation using selection derivatives, a known concept, which models the way AD differentiates nonsmooth programs. We restrict it to certain classes of elementary functions, opening the possibility to use powerful geometric techniques.

Further questions include convergence rates and complexity issues, hardly tackled at this day, let us mention the attempt of \cite{zhang2020complexity}. Our theory is limited to continuous functions and an interesting venue is to extend it to discontinuous functions, in view of treating ranking operations \cite{blondel2020fast} ubiquitous in recommendation systems, or more generally differentiating through an argmax \cite{berthet2020learning}.

\section*{Broader impact}
One of the goals of the paper is to raise awareness about an important issue of in the training of ML methods: the spuriousness of AD. To address adequately this issue, we think it is necessary  to include algorithmic differentiation explicitly in the study of optimization algorithms, a point of view which is largely ignored by today's machine learning community. 

\begin{ack}
The authors acknowledge the support of ANR-3IA Artificial and Natural Intelligence Toulouse Institute, Air Force Office of Scientific Research, Air Force Material Command, USAF, under grant numbers FA9550-19-1-7026, FA9550-18-1-0226, and ANR MasDol. J. Bolte acknowledges the support of ANR Chess, grant ANR-17-EURE-0010 and ANR OMS. The authors would like to thank anonymous referees for careful reading of this work and useful suggestions.
The authors would like to thank N. Asher and S. Gerchinovitz for useful discussions. We also warmly thank J. Malick who triggered this research. 
\end{ack}

\newpage
\appendix

This is the appendix for ``A mathematical model for automatic differentiation in machine learning''.

\section{A more comprehensive discussion and auxiliary results}
\label{sec:moreDetails}
\subsection{Related work and contribution}
The use of backward mode of algorithmic differentiation (AD) for neural network training expanded in the 80's, the most cited reference being \cite{rumelhart1986learning}. However the theory applies to much more optimization problems, see for example \cite{griewank2012who}. Indeed, numerical libraries implementing the backward mode of AD were already available in the 90's for \texttt{FORTRAN} code \cite{bischof1992adifor,bischof1996adifor} or \texttt{C/C++} code \cite{griewank1996algorithm}, 30 years before the emergence of \texttt{python} libraries. These early implementation could differentiate virtually any code, but their domain of validity, i.e., the setting for which one could predict what the output would be, was restricted to differentiable functions evaluated on their (open) domain of differentiability. 

This was well known to the AD community, see for example \cite{griewank2008evaluating}, and exploring further the domain of validity of AD, beyond mere differentiability, was already a vivid problem.

Let us mention \cite{griewank2008evaluating} who used notions such as finite selection, ``isolated criticalities", stable domain or regular arcs, and argued that ``functions given by evaluation procedures are almost everywhere real analytic or stably undefined'' where ``undefined'' meant that a nonsmooth elementary function is used in the evaluation process. 
For piecewise smooth functions which nonsmoothness can be described using the absolute value function (abs-normal form), \cite{griewank2013stable} developped a piecewis linearisation formalism and local approximation related to AD, \cite{griewank2016lipschitz} proposed an AD based bundle type method. These developments are based on the notion of piecewise smooth functions \cite{scholtes2012introduction} which we use in this work. More recently, \cite{griewank2019treating} applied these techniques to single layer neural network training and \cite{griewank2020beyond} proposed to avoid the usage of subgradient ``oracles'' in nonsmooth analysis as they are not available in practice.
In a similar vein, let us mention \cite{barton2018computationally} study lexicographic derivatives, a notion of directional derivatives which satisfy a chain rule making them compatible by forward mode AD, and \cite{zhang2020complexity} who use directional derivatives in the context of local sampling stochastic approximation algorithms for machine learning.

Constraint qualification is known in nonsmooth analysis to ensure favorable behavior of chain rules of differential calculus for nonsmooth objects (see \cite{rockafellar1998Variational}). These already appeared in the context of piecewise smooth functions of Scholtes with the notion of ``essential selections''. Such an approach was used in \cite{kakade2018provably} to propose an AD algorithm for subgradient computation under constrant qualification. Similarly \cite{griewank2016first} study first and second order optimality, in relation to AD using constraint qualification.

The current work departs from all these approaches in a fundamental way. We propose to study backward mode of AD, as implemented for nonsmooth functions by standard software (e.g. TensorFlow, PyTorch), without any modification, addition of operations or hypotheses. Our theoretical results model AD as implemented in current machine learning libraries. Contrary to previous works, our focus is precisely on the unpredictable behavior of AD in nonsmooth context. Our main contribution is to show that in a stochastic optimization context, this spurious behavior is essentially harmless from a theoretical point of view, providing justifications for the use of AD outside of its original domain of validity in machine learning. 

At the time this paper was accepted, we learnt about a paper proposing an analysis close to ours \cite{lee2020correctness}. The authors show that AD applied to programs involving piecewise analytic continuous functions, under analytic partitions, compute gradients almost everywhere. This is the counterpart of Proposition~\ref{gevery}, replacing log-exp elementary function in Definitions \ref{def:elemFunction} and \ref{def:elemIndex}, by analytic functions.

\subsection{Implementation of $\relu$}
\label{sec:suppRelu}
The implementation of the $relu$ function used in Figure \ref{fig:illustrAutodiff} is given by the function \texttt{tf.nn.relu} in Tensorflow software library \cite{abadi2016tensorflow}. This implementation corresponds to the selection function described in Section \ref{sec:basicPDfunction} and the same result may be obtained by an explicit implementation of this branching selection as illustrated in the following figure

\begin{center}
				\includegraphics[width=.6\textwidth]{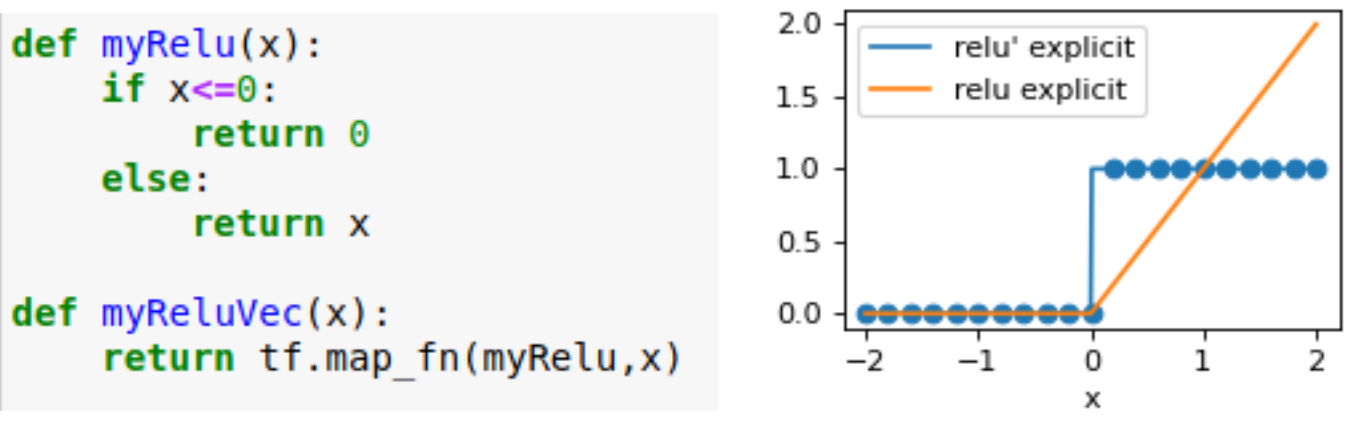}
\end{center}

One can imagine an equivalent implementation of $relu$ with a slightly different branching involving a strict inequality, that would correspond to an equivalent implementation of the same function, but the computed derivative at $0$ is different due to the implementation

\begin{center}
				\includegraphics[width=.6\textwidth]{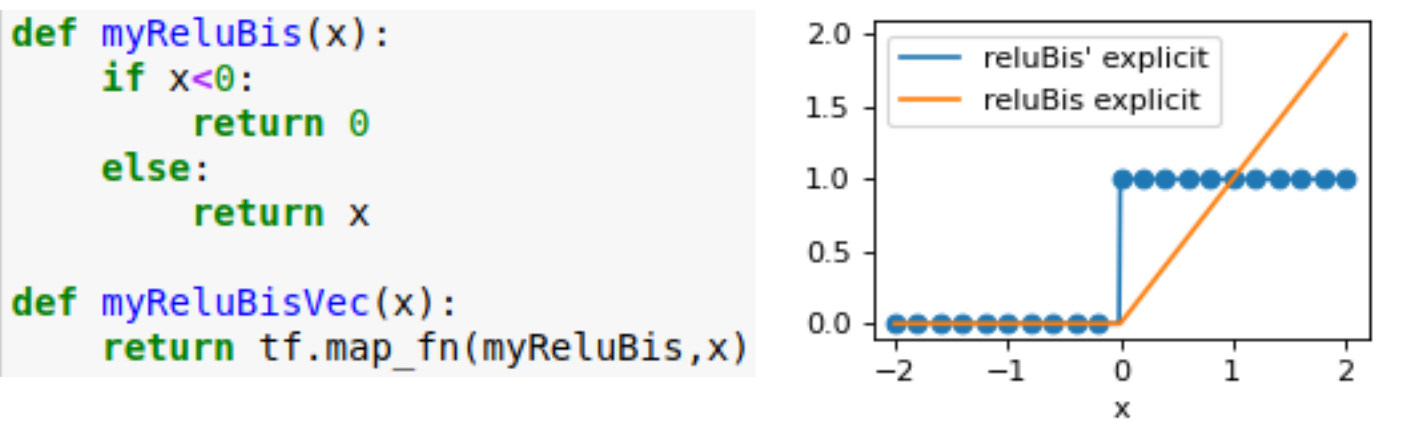}
\end{center}

\subsection{Auxiliary results and remarks}
\label{sec:auxiliaryResults}

\begin{remark}[Elementary piecewise differentiable functions]\label{r:elementary}{\hfill\\\rm (a) The building blocks in the construction of $\SSS$ in Definition \ref{def:C1selection} could be modified and adapted to other needs. Besides, the results we present in this article would remain true if we added real analytic functions restricted to compact sets.\\
(b) Note also that in Definition \ref{def:C1selection}, functions are actually real analytic on their (open) domain of definition. Yet their extension might not be analytic, as for instance the function $f:x\neq0\to\exp(-1/x^2)$ extended by $f(0)=0$.\\
(c) The construction of elementary piecewise functions in Definition \ref{def:C1selection}, does not coincide in general with some natural minimal o-minimal, but are contained in a larger such structure. For instance, when the basic bricks are polynomial functions, we obtain the field of rational functions which differs from the set of semi-algebraic functions.}
 \end{remark}

\begin{proposition}[$D_f$ has a closed graph]
				As $k\to \infty$, assume that $x_k \to \bar{x} \in \RR^p$ and $v_k \in D_f(x_k)$, $v_k \to \bar{v}$. Then $\bar{v} \in D(\bar{x})$.
				\label{prop:Dcontinuity}
\end{proposition}

\section{Proofs}
\label{sec:proofs}

\begin{proof}[of Theorem \ref{th:noOperator}]
Recall the operator is denoted by $\partial^{A}$. Fix a function $f$, by point (a), the operator  $\partial^Af$ should contain
$$\left\{\begin{aligned}
\R^p& \rightrightarrows&& \R^p\\
x& \to &&\left\{\A(P)(x): \FF(P)=f, P\in \PP\right\}
\end{aligned}\right.$$
Let us show that the graph of the above is $\R^p\times\R^p$. Assume $p=1$ for simplicity. For real numbers $r,s$, consider the functions $f_{r,s}=f+r\,\zero(\cdot-s)$ which coincide with $f$ but whose form induces programs $P_{r,s}$ of $f$. These satisfy $\FF(P_{r,s})=f$ and $\A(P_{r,s})(s) \ni \A(f)(s)+r$. Since $r$ is arbitrary, $\partial^A f(s) = \RR^p$ and since $s$ is arbitrary, we actually have
$$\mbox{graph}\,\partial^A f=\R^p\times \R^p.$$
Since $f$ is arbitrary, we have shown that $\partial^A$ is trivial.
\end{proof}

\begin{proof}[of Proposition \ref{prop:SCintegration}]
                The proposition is a consequence of Theorem \ref{th:tensorflowConservativeField} and \eqref{eq:conservativity} but it admits a more elementary proof which we detail here.
				 Fix $x,y \in \RR^p$. Let us admit the following claim --whose independent proof is given in Section \ref{sec:ominimal}.
				\begin{claim}
								There exists a finite set of numbers $0 = a_0 < a_1 <  \ldots < a_N = 1$, such that for all $i \in 0,\ldots N-1$, the function $t \mapsto s(x + t(y-x))$ is constant.
								\label{cl:finitePartitionSegment}
				\end{claim}
				Fix $i \in 0 \ldots, N-1$, and $j \in 1 \ldots m$ such that $f = f_j$ on $(x + a_i(y-x), x + a_{i+1}(y-x))$. Since $f_j \in \mathcal{E}_p$, it is $C^1$ and we have by the fundamental theorem of integral calculus
				\begin{align*}
								f(x + a_{i+1}(y-x)) - f(x + a_{i}(y-x)) &= \int_{a_i}^{a_{i+1}} \left\langle \nabla f_j(x + t(y-x)),y-x \right\rangle dt\\
								&= \int_{a_i}^{a_{i+1}} \left\langle \selgrad f(x + t(y-x)),y-x \right\rangle dt.
				\end{align*}
				The conclusion follows because
				\begin{align*}
								f(y) - f(x) &= \sum_{i = 0}^{N-1} f(x + a_{i+1}(y-x)) - f(x + a_{i}(y-x)) \\
								&= \sum_{i = 0}^{N-1}\int_{a_i}^{a_{i+1}} \left\langle \selgrad f(x + t(y-x)), y-x \right\rangle dt \\
								&= \int_{0}^1 \left\langle \selgrad f(x + t(y-x)),y-x \right\rangle dt.
				\end{align*}
\end{proof}

\begin{proof}[of Proposition \ref{prop:SGradientAE}]
                Constructs the sets $U_i$ by considering sets $V_j = \left\{ x \in \RR^p,\, s(x) = j \right\}$, $j = 1 \ldots m$, the proof of the following claim is postponed to Section \ref{sec:ominimal}.
                \begin{claim}
							    The boundary of each $V_j$ has zero measure and $\mathrm{cl}\left(\cup_{i=j}^m \mathrm{int}(V_j) \right) = \RR^p$.
								\label{cl:finitePartition}
				\end{claim}
			    Hence, we may define $U_1,\ldots, U_N$ by keeping only those sets with nonempty interior and take their closure. On each set $U_i$, $f$ is identical to $f_k$ for some $k$  and the result follows.
\end{proof}

\begin{lemma}
    Let $t \in \mathcal{I}$ be an elementary index on $\RR^{p_2}$ and $F \colon \RR^{p_1} \mapsto \RR^{p_2}$ with each coordinate in $\mathcal{E}$, then $t \circ F$ is an elementary index on $\RR^{p_1}$. 
    \label{lem:elementaryIndex}
\end{lemma}
\begin{proof}
    Fix an arbitrary integer $i$ in the image of $t$, by Definition \ref{def:elemIndex}, there exists elementary functions $h_1,\ldots,h_J$, $J \in \NN$ on $\RR^{p_2}$ such that $t(y) = i$ if and only if $y \in K_i := \{z \in \RR^{p_2},\, h_j(z) \,\diamond_j\, 0,\, j= 1,\ldots J\}$ where $\diamond_j$ is an equality or inequality sign depending on $j$. Then $t(F(x)) = i$ if and only if $F(x) \in K_i$ which is equivalent to say that $x \in \tilde{K}_i:= \{x \in \RR^{p_1}, h_j(F(x)) \,\diamond_j\, 0,\, j = 1,\ldots J \}$. By Definition \ref{def:elemFunction}, $h_j \circ F$ is an elementary function for $j = 1,\ldots,J$ and $i$ was an arbitrary integer, this shows that we have an elementary index.
\end{proof}

\begin{proof}[of Proposition \ref{p:stability}]
					Let $F \colon \RR^{p_1} \mapsto \RR^{p_2}$ such that each of its coordinate $f_i$, $i = 1 \ldots p_2$, is in $\mathcal{S}$ and $g \colon \RR^{p_2} \mapsto \RR$, $g \in \mathcal{S}$. We establish that $g\circ F$ is an elementary selection, the other cases are similar. We may consider all possible intersections of constant index domains across all coordinates of $F$ in $\left\{ 1,\ldots,p_2 \right\}$. We obtain $(s, F_1, \ldots, F_m)$, an elementary selection for $F$ (each $F_i \colon \RR^{p_1} \mapsto \RR^{p_2}$ has coordinates in $\mathcal{E}$) . Consider $g \in \mathcal{S}$ with elementary selection $(t, g_1, \ldots, g_l)$. The composition $g \circ F$ may be written as
				\begin{align*}
								g(F(x)) = g_{t(F(x))}(F(x)) = g_{t(F_{s(x)}(x))}(F_{s(x)}(x)).
				\end{align*}
				For each $i = 1 \ldots,m$ and $j = 1,\ldots, l$, consider the set
				\begin{align*}
								U_{ij} = \left\{ x \in \RR^p,\, s(x) = i, t(F_{i}(x)) = j \right\}.
				\end{align*}
				Fix $(i,j)$ in $\{1,\ldots,m\}\times\{1,\ldots,l\}$, by Lemma \ref{lem:elementaryIndex}, $t \circ F_i$ is an elementary index on $\RR^{p_1}$.
				Hence  $U_{ij}$
				is the solution set of finitely many equalities and inequalities involving functions in $\mathcal{E}$. We associate to the bi-index $(i,j)$ the corresponding set $U_{ij}$ and the function $g_j(F_i(x)) \in \mathcal{E}$. Note that we assumed that the composition is well defined. Identifying each pair $(i,j)$ with a number in $\left\{ 1,\ldots, n m \right\}$, we obtain an elementary selection for $g \circ F$ and hence $g \circ F \in \mathcal{S}$.
\end{proof}

\begin{proof}[of Proposition \ref{prop:compositionGeneralizedDerivatives}]
				The derivation formula follows from the  proof argument of Proposition~\ref{p:stability}, for each pair $(i,j)$, the function $g_j \circ F_i$ is the composition of two $C^1$ functions and its gradient is given by $J_{F_i} \times \nabla g_j \circ F_i$ on $U_{ij}$. By construction of $U_{ij}$ and definition of the selection derivative, this corresponds to \eqref{eq:chainRule} on $U_{ij}$ and the result follows.
\end{proof}

\begin{proof}[of Lemma \ref{lem:algebraBackprop}]
   We actually prove a slightly stronger result, namely for each $i\in\{p+1,\ldots,m-1\}$
		\begin{align}
					P_i \left(I - e_{i+1} e_{i+1}^T + d_{i+1}e_{i+1}^T \right) \ldots \left(I - e_{m} e_{m}^T + d_{m}e_{m}^T \right) = P_i \left(I + d_{i+1}e_{i+1}^T \right) \ldots \left(I + d_{m}e_{m}^T \right) 
					\label{eq:backpropAlgebra}
		\end{align}
		
		We argue by exhaustion from $i=m-1$ downward to $i=p$, which is the result of interest. If $i= m-1$, we indeed have
		\begin{align*}
					P_{m-1}\left(I - e_{m} e_{m}^T + d_{m}e_{m}^T \right) = P_{m-1}  \left(I + d_{m}e_{m}^T \right) 
		\end{align*}
		since $P_{m-1} e_me_m^T  = 0$. Now assume that \eqref{eq:backpropAlgebra} holds true for an index $i$ within $\{p+1,\ldots,m-1\}$, then we have
				\begin{align*}
					&P_{i-1} \left(I - e_{i} e_{i}^T + d_{i}e_{i}^T \right) \ldots \left(I - e_{m} e_{m}^T + d_{m}e_{m}^T \right) \\ 
					=\quad & P_{i-1} \left(I - e_{i} e_{i}^T + d_{i}e_{i}^T \right) \left(I - e_{i+1} e_{i+1}^T + d_{i+1}e_{i+1}^T \right) \ldots \left(I - e_{m} e_{m}^T + d_{m}e_{m}^T \right) \\
					=\quad & P_{i-1} \left(I - e_{i} e_{i}^T+ d_{i}e_{i}^T \right) P_{i}\left(I - e_{i+1} e_{i+1}^T + d_{i+1}e_{i+1}^T \right) \ldots \left(I - e_{m} e_{m}^T + d_{m}e_{m}^T \right) \\
					=\quad & P_{i-1} \left(I + d_{i}e_{i}^T \right) P_{i} \left(I + d_{i+1}e_{i+1}^T \right) \ldots \left(I + d_{m}e_{m}^T \right) \\
					=\quad & P_{i-1} \left(I + d_{i}e_{i}^T \right)\left(I + d_{i+1}e_{i+1}^T \right) \ldots \left(I + d_{m}e_{m}^T \right),
		\end{align*}
		where step 1 is expanding the product, step 2 is because $P_{i-1} P_i = P_{i-1}$ and $e_i^T P_i = e_i^T$, step 3 combines the fact that $P_{i-1}e_i = 0$ and \eqref{eq:backpropAlgebra} which we assumed to be true, the last step uses again the fact that $P_{i-1} P_i = P_{i-1}$ and $e_i^T P_i = e_i^T$. Hence  the result holds by exhaustion.
\end{proof}

\begin{proof}[of Proposition \ref{prop:Dcontinuity}]
				Consider the sequence $s_k = S(x_k)$, by taking a subsequence we may assume that $s_k$ is constant, say equal to $\left\{1,\ldots,r \right\}$. Hence for all $k$, $v_k \in \conv\left( \left\{ \nabla f_i(x_k),\, i =1 ,\ldots r \right\} \right)$ and $f(x_k) = f_i(x_k)$, $i = 1, \ldots, r$. Passing to the limit, we have $f(\bar{x}) = f_i(\bar{x})$, $i = 1, \ldots, r$ and hence $\left\{ 1,\ldots,r \right\} \in S(x)$. Furthermore, $\bar{v} \in \mathrm{conv}\left( \left\{ \nabla f_i(\bar{x}),\, i =1 ,\ldots r  \right\} \right)\subset D_f(\bar{x})$.
\end{proof}

\section{o-minimal structures, definability and conservative fields}
\label{sec:ominimal}
\subsection{$(\RR,\exp)$-definability}
We recall here the  results of  geometry  that we use in the present work.
Some references on this topic are \cite{Cos99,dries1996geometric}.

  An {\em o-minimal structure} on $(\R,+,\cdot)$ is a collection of sets
  $\mathcal{O} = (\mathcal{O}_p)_{p \in \NN}$ where each $\mathcal{O}_p$ is itself a family of
	subsets of $\R^p$, such that for each $p \in \NN$:
  \begin{enumerate}
  \item[(i)] $\mathcal{O}_p$ is stable by complementation, finite union, finite intersection and contains $\R^p$.
    \item[(ii)]  if $A$ belongs to $\mathcal{O}_p$, then both $A \times \R$ and $\R \times A$
      belong to $\mathcal{O}_{p+1}$;
    \item[(iii)]  if $\pi: \R^{p+1} \to \R^p$ is the canonical projection onto $\R^p$ then,
      for any $A \in \mathcal{O}_{p+1}$, the set $\pi(A)$ belongs to $\mathcal{O}_p$;
      \label{it:algebraic}
    \item[(iv)]  $\mathcal{O}_p$ contains the family of real algebraic subsets of $\R^p$, that is,
      every set of the form
      \[
        \{ x \in \R^p \mid g(x) = 0 \}
      \]
      where $g: \R^p \to \R$ is a polynomial function;
    \item[(v)]  the elements of $\mathcal{O}_1$ are exactly the finite unions of  intervals.
  \end{enumerate}

A subset of $\R^p$ which belongs to an o-minimal structure $\mathcal{O}$ is said to be
{\em definable in $\mathcal{O}$}. A function is {\em definable in $\mathcal{O}$} whenever its graph is definable in $\mathcal{O}$). A set valued mapping (or a function) is said to be definable in $\mathcal{O}$ whenever its graph is definable in $\mathcal{O}$. The terminology {\em tame} refers to definability in an o-minimal structure without specifying which structure.

  The simplest  o-minimal structure is given by the class of
  real  semialgebraic objects.
  Recall that a set $A \subset \R^p$ is called {\em semialgebraic} if it is a finite union of sets of the form $$\displaystyle  \bigcap_{i=1}^k \{x \in \R^p \mid g_{i}(x) < 0, \; h_{i}(x) = 0 \}$$
  where the functions $g_{i}, h_{i}: \R^p \to \R$ are real polynomial functions and $k\geq 1$.
  The key tool to show that these sets form an o-minimal structure is Tarski-Seidenberg
  principle which ensures that (iii) holds true. 
  
  According to \cite{wilkie1999theorem}, there is an o-minimal structure which contains all semialgebraic sets and the graph of the exponential function, we fix this o-minimal structure and call it $\mathcal{O}$. As a consequence, all functions which can be described by a finite compositional expression involving polynomials, quotients, exponential and logarithms are definable in $\mathcal{O}$. In particular any function $f \in \mathcal{S}$ is definable in $\mathcal{O}$, which opens the use of powerful geometric tools \cite{Cos99,dries1996geometric} for functions in $\mathcal{S}$. From now on, we call an object \emph{definable} if it is definable in $\mathcal{O}$.

 As detailed in \cite{Cos99} the following holds true
 \begin{proposition}[Quantifier elimination]
              Any first order formula (quantification on variables only) involving definable functions and definable sets describes a definable set.
 \end{proposition}
This allows to prove Claim \ref{cl:finitePartitionSegment}

\begin{proof}[of Claim \ref{cl:finitePartitionSegment}]
    The function $t \mapsto s(x + t(y-x))$ is definable and has values in $\{1,\ldots, m\}$. For each $j \in \{1,\ldots, m\}$, the set $S_j = \{t \in [0,1],\, s(x + t(y-x)) = j\}$ is definable, and by (v), it is a finite union of intervals. For each $j$ consider only the endpoints of those intervals with nonempty interior, this provides the desired partition.
\end{proof}

\subsection{Properties of definable sets}

	The tangent space at a point $x$ of a manifold $M$ is denoted by $T_xM$. Given a submanifold\footnote{We only consider embedded submanifolds} $M$ of a finite dimensional Riemannian manifold, it is endowed by the Riemanninan structure inherited from the ambient space. Given $f \colon \RR^p \to \RR$ and $M\subset\R^p$ a differentiable submanifold  on which $f$ is differentiable, we denote by $\mbox{grad}_M f$ its Riemannian gradient or even, when no confusion is possible,  $\grad f$.

 A $C^r$ stratification of a (sub)manifold $M$ (of $\R^p$) is a partition $\SSS=(M_1,\ldots,M_m)$ of $M$ into $C^r$ manifolds having the property that $\cl M_i\cap M_j\neq \emptyset$ implies that $M_j$ is entirely contained in the boundary of $M_i$ whenever $i\neq j$. Assume that a function $f:M \to \R$ is given and that $M$ is stratified into manifolds on which $f$ is differentiable. For $x$ in $M$, we denote by $M_x$ the strata containing $x$ and we simply write $\grad f(x)$ for the gradient of $f$ with respect to $M_x$.

Stratifications can have many properties, we refer to \cite{dries1996geometric} and references therein for an account on this question and in particular for more on the idea  of a Whitney stratification that we will use repeatedly. 
We pertain here to one basic definition: a $C^{r}$-stratification $\SSS=(M_{i})_{i\in I}$ of a manifold $M$ has the
\emph{Whitney-}($a$)\emph{ property,} if for each $x\in\cl M_{i}\cap
M_{j}$ (with $i\neq j$) and for each sequence $(x_{k})_{k\in\NN}\subset
M_{i}$ we have:
\[
\left.
\begin{array}
[c]{ll}
& \underset{k\rightarrow\infty}{\lim}\mathcal{\;}x_{k}\mathcal{\;}=x\\
& \text{}  \\
&
\underset{k\rightarrow\infty}{\lim}\mathcal{\;}T_{x_{k}}M_{i}\mathcal{\;}
=\mathcal{T}
\end{array}
\right\}  \mathcal{\;}\Longrightarrow\mathcal{\;}T_{x}M_{j}\mathcal{\;}
\subset\mathcal{\;T}
\]
where the second limit is to be understood in the Grassmanian, i.e., ``directional", sense. In the
sequel we shall use the term \emph{Whitney stratification} to refer to a
$C^{1}$-stratification with the Whitney-($a$) property. The following can be found for example in \cite{dries1996geometric}.

\begin{theorem}[Whitney stratification]
    Let $A_1,\ldots, A_k$ be definable subsets of $\RR^p$, then there exists a definable Whitney stratification $(M_{i})_{i\in I}$ compatible with $A_1,\ldots, A_k$, \emph{i.e.} such that for each $i \in I$, there is $t \in \left\{1 , \ldots k\right\}$, such that $M_i \subset A_t$.
    \label{th:wstratification}
\end{theorem}

This allows for example to prove Claim \ref{cl:finitePartition}

\begin{proof}[of Claim \ref{cl:finitePartition}]
    The sets $V_1,\ldots, V_m$ form a definable partition of $\RR^p$. Consider a Whitney stratification of $\RR^p$, $(M_i)_{i \in I}$ compatible with the closure of $V_1,\ldots, V_m$. The boundary of each $V_i$ is a finite union of strata of dimension strictly smaller than $p$ and hence has measure zero. The remaining strata (open of maximal dimension) have to be dense in $\RR^p$ since we started with a partition.
\end{proof}

\subsection{Variational stratification and projection formulas}

\begin{definition}[Variational stratification]{\rm 
				Let $f \colon \RR^p \to \RR$, be locally Lipschitz con\-ti\-nuous, let $D \colon \RR^p \rightrightarrows \RR^p$ be a set valued map and let $r\geq 1$. We say that the couple $(f,D)$ has a $C^r$ {\em variational stratification} if there exists a $C^r$ Whitney stratification $\SSS = (M_i)_{i \in I}$ of $\RR^p$, such that $f$ is $C^r$ on each stratum  and for all $x \in \R^p$,
				\begin{align}\label{projf}
								\mathrm{Proj}_{T_{M_x}(x)} D(x) = \left\{ \grad f(x) \right\},
				\end{align}				
				 where $\grad f(x)$ is the gradient of $f$ restricted to the active strata $M_x$ containing $x$.	}	\label{def:projectionFormula}
\end{definition}

The equations \eqref{projf} are called \emph{projection formulas} and are motivated by Corollary 9 in \cite{bolte2007clarke} which  states that Clarke subgradients of definable functions have projection formulas.

Let us recall the definition of conservative set-valued mappings from \cite{bolte2020conservative} and one of its characterization.

\begin{definition}[Conservative set-valued mappings] {\rm Let $f$ be a Lipschitz continuous function. A set valued vector field $D$ is called {\em conservative} if for any absolutely continuous path $\gamma \colon [0,1] \mapsto \RR$, we have
\begin{align}
				f(\gamma(1)) - f(\gamma(0)) = \int_{0}^1 \min_{v \in D(\gamma(t))}\left\langle v, \dot{\gamma}(t) \right\rangle dt =  \int_{0}^1 \max_{v \in D(\gamma(t))}\left\langle v, \dot{\gamma}(t) \right\rangle dt .
				\label{eq:conservativity}
\end{align}}
\end{definition}
Equivalently $D$ is conservative for $f$, if for all absolutely continuous curves $\gamma \colon [0,1] \mapsto \RR^p$, for almost all $t \in [0,1]$, $f\circ \gamma$ is differentiable and
				\begin{align*}
								\frac{d}{dt} f(\gamma(t)) = \left\langle v, \dot{\gamma}(t) \right\rangle,\qquad \forall v \in D(\gamma(t)).
				\end{align*}

The following combines other results from \cite{bolte2020conservative}, where one implication is essentially due to \cite{davis2018stochastic} based on \cite{bolte2007clarke}.

\begin{theorem}[Characterization of conservativity]
    Let $D \colon \RR^p \rightrightarrows \RR^p$ be a definable, nonempty compact valued, graph closed set valued field and $f \colon \RR^p \mapsto \RR$ be a definable locally Lipschitz function. Then the following are equivalent
    \begin{itemize}
        \item $D$ is conservative for $f$.
        \item For any $r\geq 1$, $(f, D)$ admit a $C^r$ variational stratification.
    \end{itemize}
    \label{th:characterizationConservativity}
\end{theorem}
This result allows to prove the following

\begin{proof}[of Theorem \ref{th:tensorflowConservativeField}]
				We prove that there is a $C^1$ projection formula (see Theorem \ref{th:characterizationConservativity}). For each $I \subset \left\{ 1,\ldots,m \right\}$, set $V_I = \left\{ x \in \RR^p, S(x) = I \right\}$. On each set $V_I$, $f(x) = f_i(x)$ for all $i \in I$. These sets are definable, hence, there is a definable Whitney stratification of $\RR^p$ which is compatible with them (Theorem \ref{th:wstratification}). For any $C^1$ manifold $M$ in the stratification there is an index set $I \subset \{1,\ldots,m\}$ such that for all $i \in I$ and all $x \in M$, $f(x) = f_i(x)$ and $S(x) = I$. Since each $f_i$, $i \in I$ is $C^1$ and they agree on $M$, they represent the same function when restricted to $M$. Hence they have the same differential on $M$ and since they are all globally $C^1$ this agrees with the projection of their gradient on the tangent space of $M$. Hence the projection of $D_f(x)$ to the tangent space to $M$ at $x$ is single valued and corresponds to the derivative of $f$ restricted to $M$. This is sufficient to conclude as this is precisely the variational stratification required by Theorem \ref{th:characterizationConservativity}.
\end{proof}

\section{Convergence to selection critical points}

\begin{proof}[of Theorem \ref{th:convergenceGrad}, first part]
We use here the results on conservative fields developed in \cite{bolte2020conservative}. To prove the theorem it suffices to establish that:
\begin{itemize}
    \item $D_J$ is a conservative field for $J$
    \item the number of $D_J$ critical values are finite.
\end{itemize}
The first point is Theorem \ref{th:characterizationConservativity} while the second one is the consequence of the latter and the definability of the couple $f,D_f$, see Proposition~\ref{p:art} (ii). To conclude it suffices to apply the convergence results in \cite[Theorem 9]{bolte2020conservative}.
\end{proof}

\begin{proof}[of Theorem \ref{th:trapAvoidance}, second part]
		This result is a consequence of the more general Theorem~\ref{t:trap} established in Section~\ref{sec:art}. Let $F$ be the finite set given in Theorem \ref{t:trap}, the set $$\{c \in (0,1],\, \exists k \in \NN,\, c \gamma_k \in F\},$$ is countable, and hence has zero measure. So for almost all $c \in (0,1]$, $\{c \gamma_k\}_{k \in \NN}$ does not intersect $F$. Using Theorem \ref{t:trap}, there is a zero measure set $N$ such that any initialization outside $N$ provides almost surely a subgradient sequence. By hypothesis, for almost every $x_0 \in K \setminus N$, the sequence is bounded almost surely and the result follows from Theorem \ref{t:trap}.
\end{proof}

\section{Artificial critical points}\label{sec:art}

 Being given a Lipschitz continuous function on $\R^p$ and a conservative field $D$, one has two types of $D$-critical points:
\begin{itemize} 
\item Clarke critical points: $\partial^c f(x)\ni0$, we denote the set of these points by $\crit^cf $
\item  Artificial critical points $\partial^c f(x)\not\ni0$ and $D(x)\ni 0$, we denote this set by 
$\crit^a f$
\end{itemize}

Critical values are defined accordingly as images of critical points.

\begin{proposition}[Artificial critical points]\label{p:art} Assume $f:\R^p\to \R$ and $D:\R^p\rightrightarrows\R^p$ are definable in a common o-minimal structure.  
             The connected components $C_i$  of $\crit^a f$, which are in finite number, satisfy
             \begin{itemize}
                 \item[(i)] $\dim C_i<p$
                 \item[(ii)] $f(C_i)$ is a singleton, and as a consequence the $D$ critical values of $f$ are in finite number,
                 \item[(iii)] $\crit^a f$ does not contain local minimum (nor local maximum)
             \end{itemize}
\end{proposition}
\begin{proof} By definability of $\crit^a f$, the number of connected components is finite.

If $C_i$ had full dimension it would contain a non trivial ball on which $f$ should be constant by the integral property. This would in turn imply that the points in the ball would also be local minimum and thus Clarke critical, which is impossible.

To see that the critical values are in finite number it suffices to evoke the fact that Clarke critical values are finite \cite{bolte2007clarke} and use that artificial critical values are in finite number.

By definability the connected components are arcwise-connected with piecewise $C^1$ paths. Using the integral property this shows $f$ is constant on $C_i$. 

(iii) is obvious since local minimum or maximum are Clarke critical.
\end{proof}

As explained in the introduction, artificial critical points are ``computing artefacts", whence their names. For algorithmic differentiation the ``gradient" provided by a program is zero while the point might even be a smooth non critical point. 
We consider the setting of the mini-batch algorithm of the last section. 

\begin{theorem} \label{t:trap}
Assume that each $f_1,\ldots,f_n$ belongs to $\SSS$. 
There exists a finite subset of steps $F\subset (0,+\infty)$ and a zero measure meager subset $N$ of $\R^p$, such that for any positive sequence $\gamma_k=o(1/\log k)$ avoiding values in $F$, and any almost surely bounded sequence with initial condition in $\R^p\setminus N$, we have 
\begin{itemize}
\item $J(x^k)$ converges towards a Clarke critical value almost surely,
\item the cluster points of $x^k$ are Clarke critical point almost surely.
\end{itemize}
\end{theorem}
\begin{proof} The proof is twofold. We first prove that the set of initial conditions leading to an artificial critical point or more generally to a non differentiability point within a finite time is ``small". We then use this fact to conclude.
\begin{claim}\label{cl:nonflat} Let $g:\R^p\to \R$ be a definable differentiable function. Set, for $\lambda>0$, 
$$\Phi_{\lambda}=\lambda Id-\nabla g,$$
where $Id$ denotes the identity. There exists a finite set $F$ in $(0,+\infty)$ such that, 
\begin{equation}\forall \lambda \in (0,+\infty)\setminus F, \:\forall Z\subset \R^p \mbox{ definable }, \dim Z<p \Rightarrow  \dim \Phi_{\lambda}^{-1}(Z)<p.\label{e:dim}\end{equation}
\end{claim}
Proof of the claim. Denote by $L$ the set of points where $g$ is twice differentiable so that $L$ is dense and definable. Denote by $\lambda_1,\ldots,\lambda_p:L\to\R$ a representation of the eigenvalues of $\nabla^2g$. Refine $L$ to be contained in the common domain of differentiability for each $\lambda_i$, $L$ remains open and dense. By the definable Sard's theorem the critical values of each function $\lambda_i$ is finite, so that the set of all these values which we denote by $F$ is itself finite. 

Take a positive real $\lambda\notin F$ and consider the set $$K_\lambda:=\{x\in L:\Phi'_{\lambda}(x)=\lambda Id-\nabla^2g(x) \mbox{ is not invertible}\}.$$ By diagonalization, we see that the determinant of $\Phi'_{\lambda}(x)$ is $\displaystyle \prod_{i=1}^d (\lambda-\lambda_i(x))$ for any $x$, thence 
$$ K_{\lambda}\subset \bigcup_{i=1}^m\{x \in L,\, \lambda_i(x)=\lambda\}.$$
Since $\lambda$ is a regular value for each $\lambda_i$ the previous set is a finite union of manifolds of dimension $p-1$, see e.g., \cite{Cos99}. This implies that the set $\R^p\setminus K_{\lambda}=\left\{x\in L: \Phi'_{\lambda}(x) \mbox{ is invertible }\right\}$ 
is dense. 
Using the above, we deduce that there exists finitely many open connected subsets $U_1,\ldots,U_r\subset L$ of $\R^p\setminus K_{\lambda}$ such that $\displaystyle U_1\cup\ldots\cup U_r$ is dense in $L$ and thus in $\R^p$. Take now $Z\subset \R^p$ definable with $\dim Z<p$. Assume towards a contradiction that there exists a nonempty open ball $B$ in $\Phi_{\lambda}^{-1}(Z)$. In that case $B$ must have a nonempty intersection with some $U_{i_0}$. The set $\Phi_{\lambda}(B\cap U_{i_0})$ is open because $\Phi_{\lambda}$ is a diffeomorphism on $U_i$ on its image. Since on the other hand we have $\Phi_{\lambda}(B\cap U_{i_0})\subset Z$, we have a contradiction and the claim is proved.$\hfill\Box$

For each $I\subset \{1,\ldots,n\}$, we denote by $f_{I,1},\ldots,f_{I,m_I}$ the bricks attached to $f_I$ where $m_I\geq 1$. Denote by $\Sing$  the set of points on which at least one $f_I$ is non differentiable and $C$ the set of points for which  $\selgrad f_I \neq \nabla f_I$ for at least one $I$. By Proposition \ref{prop:SGradientAE} and definability, $\Sing$ and $C$ are  finite unions of manifolds of dimension at most $p-1$. 

Set $\Phi^k_{I,j}=Id-\gamma_k\nabla f_{I,j}$, with $I\subset\{1,\ldots,m\}$, $j\in \{1,\ldots,m_I\}$ and $Id$ denotes the identity. 
Applying Claim \ref{cl:nonflat}, we can find a finite set $F$ for which $\gamma_k\notin F$ implies that each $\Phi^k_{I,j}$ has the property \eqref{e:dim}. Indeed, for each $I\subset\{1,\ldots,m\}$, $j\in \{1,\ldots,m_I\}$, there is $F_{I,j} \subset \RR$ finite such that $f_{I,j}$ has property \eqref{e:dim}. Since the subsets $I$ are in finite number and each $m_I$ is finite, the set $F = \bigcup_{I \subset \{1,\ldots,m\}} \bigcup_{j\in \{1,\ldots,m_I\}} F_{I,j}$, is also finite. For each $k \in \NN$, $I\subset\{1,\ldots,m\}$, $j\in \{1,\ldots,m_I\}$. Remark that if $\gamma_k \not \in F$ then $\Phi^k_{I,j}$ has property \eqref{e:dim}.

For $k\leq k_0$ fixed, let us consider the finite set of definable mappings defined by $$\Psi_{k_0}:=\left\{\prod_{j=1}^k\Phi^j_{I_j,i_j}: k\leq k_0, I_j\subset \{1,\ldots,n\}, i_j\in\{1,\ldots,m_{I_j}\} \right\}.$$
We now assume that $\gamma_k \notin F, \forall k\geq 0,$ so that each mapping in $\Psi_{k_0}$ has the property $\eqref{e:dim}$
and  $$N_{k_0}:=\left\{x\in \R^p: \exists k\leq k_0, \exists \Phi \in \Psi_{k}(x)\in C\cup\Sing \right\}$$
These are initial conditions in $U$ leading to an artificial critical or a non-differentiability point within $U$ before time $k_0$. 

We can also write   
$$N_{k_0}\subset \bigcup_{\Phi\in\Psi_{k_0}} \Phi^{-1}\left(C\cup\Sing\right).$$
From stratification arguments we know that $\Sing$ has a dimension lower than $p-1$. On the other hand, $C$ has dimension strictly lower than $p$ by Proposition \ref{prop:SGradientAE}. Claim~\ref{cl:nonflat} applies and yields $\dim \Phi^{-1}\left(C\cup\Sing\right)<p$ for all $\Phi\in \Phi_{k_0}$. As a consequence $N_{k_0}$ is closed with nonempty interior and so does $\displaystyle N:=\cup_{k\in\NN}N_k$ by Baire's theorem. Similarly $N$ has zero measure as a countable union of zero measure sets.

This proves that any sequence with initial condition out of $N$ must remain in the zone of differentiability of $J$ as well as all $f_I$. In particular if $I$ is taken uniformly at random among possible subsets, for all $x \not \in N$, we have $\mathbb{E}_I[\selgrad f_I(x)] = \selgrad J(x) = \nabla J(x) = \partial^c J(x)$, so that these specific sequences can also be seen as stochastic subgradient sequences for $J$. To be more specific, the sequence $x_k$ can be seen as one of the sequence generated by the algorithm
\begin{align*}
    y_{k+1} \in y_k - \gamma_k \partial^c J(y_k) +\epsilon_k
\end{align*}
where $\epsilon_k$ is a random noise with zero mean. Using general results \cite{davis2018stochastic,benaim2005stochastic}, we know that $y_k$ sequences, when bounded almost surely, have limit points which are Clarke critical.
\end{proof}
\end{document}